\documentclass[twoside]{article}
\usepackage[accepted]{aistats2018}

\usepackage{etoolbox}
\patchcmd{\thebibliography}{\section*{\refname}}{\subsubsection*{\refname}}{}{}

\usepackage[utf8]{inputenc}
\usepackage[T1]{fontenc}

\usepackage{mathtools}
\usepackage{amsthm}
\usepackage{bbm}
\usepackage{mathtools}
\usepackage{booktabs}

\newtheorem{theorem}{Theorem}
\newtheorem{corollary}{Corollary}
\newtheorem{lemma}{Lemma}
\newtheorem{assumption}{Assumption}

\newtheorem{defn}{Definition}

\usepackage{mathtools}

\newcommand{\expect}{\operatorname{E}\expectarg}
\DeclarePairedDelimiterX{\expectarg}[1]{[}{]}{%
  \ifnum\currentgrouptype=16 \else\begingroup\fi
  \activatebar#1
  \ifnum\currentgrouptype=16 \else\endgroup\fi
}

\newcommand{\innermid}{\nonscript\;\delimsize\vert\nonscript\;}
\newcommand{\activatebar}{%
  \begingroup\lccode`\~=`\|
  \lowercase{\endgroup\let~}\innermid 
  \mathcode`|=\string"8000
}

\allowdisplaybreaks

\DeclareMathOperator*{\argmin}{arg\,min}

\newcommand{\bs}{\boldsymbol}
\newcommand{\mr}{\mathrm}

\newcommand{\comment}[1]{}

\usepackage{float}
\usepackage[pdftex]{graphicx}
\graphicspath{{./graphics/}}
\usepackage{subfig}
\usepackage{multicol}
\usepackage{paralist}
\usepackage{bm}
\usepackage{amsfonts}
\usepackage{url}
\usepackage{multirow}
\usepackage{cases}

\DeclarePairedDelimiter{\ceil}{\lceil}{\rceil}

\usepackage{color}

\ifodd 1
\newcommand{\note}[1]{{\color{magenta} Note: #1}}
\else
\newcommand{\note}[1]{}
\fi

\ifodd 0
\newcommand{\revn}[1]{{\color{blue}#1}}
\else
\newcommand{\revn}[1]{#1}
\fi

\ifodd 0
\newcommand{\com}[1]{{\color{red} COM: #1}}
\else
\newcommand{\com}[1]{}
\fi

\ifodd 0
\newcommand{\comn}[1]{{\color{red} COMNEW: #1}}
\else
\newcommand{\comn}[1]{}
\fi

\ifodd 0
\newcommand{\rev}[1]{{\color{blue}#1}}
\else
\newcommand{\rev}[1]{#1}
\fi

\ifodd 0
\newcommand{\rvs}[1]{{\color{magenta}#1}}
\else
\newcommand{\rvs}[1]{#1}
\fi

\usepackage{pgffor, ifthen}
\newcommand{\notes}[3][\empty]{%
	\noindent \vspace{10pt}\\
	\foreach \n in {1,\ldots,#2}{%
		\ifthenelse{\equal{#1}{\empty}}
		{\rule{#3}{0.5pt}\\}
		{\rule{#3}{0.5pt}\vspace{#1}\\}
	}
}

\usepackage{algorithm,algorithmic}

\begin{document}
	
	%
	
	%
	
	\twocolumn[
	
	\aistatstitle{Multi-objective Contextual Bandit Problem with Similarity Information}
	
	
	\aistatsauthor{Eralp Tur\u{g}ay  \And Doruk \"{O}ner  \And Cem Tekin}
\aistatsaddress{Electrical and Electronics \\ Engineering Department \\ Bilkent University \\ Ankara, Turkey \And Electrical and Electronics \\ Engineering Department \\ Bilkent University \\ Ankara, Turkey \And Electrical and Electronics \\ Engineering Department \\ Bilkent University \\ Ankara, Turkey}
	]

	
	\begin{abstract}
In this paper we propose the multi-objective contextual bandit problem with similarity information. This problem extends the classical contextual bandit problem with similarity information by introducing multiple and possibly conflicting objectives. Since the best arm in each objective can be different given the context, learning the best arm based on a single objective can jeopardize the rewards obtained from the other objectives. In order to evaluate the performance of the learner in this setup, we use a performance metric called the contextual Pareto regret. Essentially, the contextual Pareto regret is the sum of the distances of the arms chosen by the learner to the context dependent Pareto front.
For this problem, we develop a new online learning algorithm called Pareto Contextual Zooming (PCZ), which exploits the idea of contextual zooming to learn the arms that are close to the Pareto front for each observed context by adaptively partitioning the joint context-arm set according to the observed rewards and locations of the  context-arm pairs selected in the past. Then, we prove that PCZ achieves $\tilde O (T^{(1+d_p)/(2+d_p)})$ Pareto regret where $d_p$ is the Pareto zooming dimension that depends on the size of the set of near-optimal context-arm pairs. Moreover, we show that this regret bound is nearly optimal by providing an almost matching $\Omega (T^{(1+d_p)/(2+d_p)})$ lower bound.
\end{abstract}
	
	\section{INTRODUCTION}
	\noindent The multi-armed bandit (MAB) is extensively used to model sequential decision making problems with uncertain rewards. 
	While many real-world applications ranging from cognitive radio networks \cite{gai2010learning} to recommender systems \cite{li2010contextual} to medical diagnosis \cite{tekin2017adaptive} require intelligent decision making mechanisms that learn from the past, majority of these applications involve side-observations that can guide the decision making process, which does not fit into the classical MAB model. 
	This issue is resolved by proposing new MAB models, called contextual bandits, that learn how to act optimally based on side-observations \cite{li2010contextual, slivkins2014contextual, lu2010contextual}. 
	On the other hand, the aforementioned real-world applications also involve multiple and possibly conflicting objectives. For instance, these objectives include throughput and reliability in a cognitive radio network, semantic match and job-seeking intent in a talent recommender system \cite{rodriguez2012multiple}, and sensitivity and specificity in medical diagnosis.

	This motivates us to develop a new MAB model that addresses the learning challenges that arise from side-observations and presence of multiple objectives at the same time. We call this model the multi-objective contextual bandit problem with similarity information. In this problem, at the beginning of each round, the learner observes a context from the context set ${\cal X}$ and selects an arm from the arm set ${\cal Y}$. At the end of the round, the learner observes a random reward whose distribution depends on the observed context and the selected arm. We assume that the sequence of contexts is fixed beforehand and does not depend on the decision of the learner. Other than this, we impose no assumptions on how the contexts are sampled from ${\cal X}$, but we assume that the expected reward of arm $y$ given context $x$ is time-invariant.
	The goal of the learner is to maximize its total reward in each objective (while ensuring some sort of fairness, which will be described later) over $T$ rounds by judiciously selecting good arms only based on the past observations and selections. To facilitate learning, we also assume that the learner is endowed with similarity information, which relates the distances between context-arm pairs to the distances between expected rewards of these pairs. This similarity information is an intrinsic property of the similarity space, which consists of all feasible context-arm pairs, and merely states that the expected reward function is Lipschitz continuous. 
    
    Different from the classical contextual bandit problem with similarity information \cite{slivkins2014contextual} in which the reward is scalar, in our problem the reward is multi-dimensional. In order to measure how well the learner performs with respect to the oracle that perfectly knows the expected reward of each context-arm pair, we adopt the notion of contextual Pareto regret defined in \cite{tekin2017multi} for two objectives, and extend it to work for an arbitrary number of objectives. The contextual Pareto regret (referred to as the Pareto regret hereafter) measures the sum of the distances of the expected rewards of the arms chosen by the learner to the Pareto front given the contexts.
    Importantly, the Pareto front can vary from context to context, which makes its complete characterization difficult even when the expected rewards of the context-arm pairs are known.
    Moreover, in many applications where sacrificing one objective over another one is disadvantageous, it is necessary to ensure that all of the Pareto optimal context-arm pairs are equally treated.
    
    We address these challenges by proposing an online learning algorithm called Pareto Contextual Zooming (PCZ) that is built on the contextual zooming algorithm in \cite{slivkins2014contextual}, and show that it achieves sublinear Pareto regret. Essentially, we provide a finite-time bound for the Pareto regret with a time order $\tilde O (T^{(1+d_p)/(2+d_p)})$, where $d_p$ is the Pareto zooming dimension, which is an optimistic version of the covering dimension that depends on the size of the set of near-optimal context-arm pairs. The proposed algorithm achieves this regret bound by adaptively partitioning the similarity space according to the empirical distribution of the past arm selections, and context and reward observations. While doing so, it only needs to keep track of a set of parameters defined for an active set of balls that cover the similarity space, and it only needs to find the balls which are not Pareto dominated by any other ball among all balls that are relevant to the current round. This shows that by making use of the similarity information, a complete characterization of the Pareto front is not necessary to achieve sublinear regret. Finally, we show an almost matching lower bound $\Omega (T^{(1+d_p)/(2+d_p)})$ based on a reduction to the classical contextual bandit problem with similarity information, which shows that our bound is tight up to logarithmic factors. 
    
    Rest of the paper is organized as follows. Related work is described in Section 2. Problem formulation is given in Section 3. The learning algorithm is proposed in Section 4, and its Pareto regret is upper bounded in Section 5. Section 6 provides a lower bound on the Pareto regret. Section 7 concludes the paper. Due to limited space, some of the technical proofs and numerical results are given in the supplemental document. 

	\section{RELATED WORK}

	We split the discussion on related work into three parts: related work in contextual bandits, related work in multi-objective bandits and related work that considers multi-objective contextual bandits. 
    
\begin{table*}[t]
\caption{Comparison with Related Work}
\label{table:related}
\centering
\resizebox{\textwidth}{!}{
\begin{tabular}{l l p{1.4cm} p{1.8cm} p{1.4cm} p{2cm} p{1.8cm}}
\toprule
Bandit algorithm & Regret bound & Multi-objective & Contextual & Linear rewards & Similarity assumption & Adaptive partition \\
\midrule
Contextual Zooming \cite{slivkins2014contextual} & $ \tilde{O}(T^{1-1/(2+d_z)})$ & No & Yes & No & Yes & Yes \\
Query-Ad-Clustering \cite{lu2010contextual} & $ \tilde{O}(T^{1-1/(2+d_c)})$ & No & Yes & No & Yes  & No \\
SupLinUCB \cite{chu2011contextual} & $\tilde{O}(\sqrt{T})$ & No & Yes & Yes & No & No\\
Pareto-UCB1 \cite{drugan2013designing} & $O(\log(T))$ (Pareto regret)  & Yes & No & No & No & No\\
Scalarized-UCB1\cite{drugan2013designing} & $O(\log(T))$ & Yes & No & No & No& No\\
MOC-MAB \cite{tekin2017multi} & $ \tilde{O}(T^{(\alpha +d)/(2\alpha +d)})$ (Pareto regret) & Yes & Yes & No & Yes & No\\
PCZ (this paper) & $ \tilde{O}(T^{1-1/(2+d_p)})$ (Pareto regret) & Yes & Yes & No & Yes & Yes \\
\bottomrule
\end{tabular}}
\end{table*}

Many different formulations of the contextual bandit problem exist in the literature, which leads to different types of algorithms and regret bounds. For instance, \cite{langford2007epoch} considers the problem where the contexts and rewards are sampled from a time-invariant distribution, and proposes the epoch greedy algorithm which achieves $O(T^{2/3})$ regret. Later on, more efficient algorithms \cite{dudik2011efficient,agarwal2014taming} that achieve $\tilde{O}(T^{1/2})$ regret are developed for this problem.

Another line of research focuses on contextual bandit problems under the linear realizability assumption, which requires the expected reward of an arm to be linear in its features (or equivalently contexts). Numerous learning algorithms are proposed for this problem including LinUCB \cite{li2010contextual} and SuplinUCB \cite{chu2011contextual} that achieve $\tilde{O}(\sqrt{Td})$ regret, where $d$ is the context dimension. SuplinUCB is also extended to work with kernel functions in \cite{valko2013finite}, and is shown to achieve $\tilde{O} (\sqrt{T \tilde{d}} )$ regret, where $\tilde{d}$ represents the effective dimension of the kernel feature space. \rev{However, these works do not take collaborative information into account. Some of the contexts may share similar behavior given the same arm and constitute clusters. In this case, the learner does not need to learn a different model for each context, but just a single model for each cluster. In \cite{gentile2014online}, CLUB uses confidence balls of the clusters (contexts are related to users in their work) to both estimate user similarity, and to learn jointly for users within the same cluster. An extension of this work \cite{li2016collaborative} considers the "two-sided clustering" where clustering process is done by simultaneously grouping arms based on similarity at the context side and contexts based on similarity at the arm side.} 

Our work builds on the formalism of contextual bandits with similarity information. The prior work in this category attempts to minimize the regret without imposing any stochastic assumptions on the context arrivals. However, knowledge of the similarity information, which relates the distances between expected rewards to the distances between the context-arm pairs in the similarity space, is required. With this assumption, Query-Ad-Clustering algorithm in \cite{lu2010contextual} achieves $O(T^{1-1/(2+d_c) +\epsilon})$ regret for any $\epsilon > 0$, where $d_c$ is the covering dimension of the similarity space. This algorithm works by partitioning the similarity space into disjoint sets and estimating the expected arm rewards for each set of the partition separately. Another related work proposes the contextual zooming algorithm \cite{slivkins2014contextual} that judiciously partitions the similarity space based on the reward structure of the similarity space and past context arrivals. It is shown that adaptive partitioning achieves better performance compared to uniform partitioning since it uses the experience gained in the past to form more accurate reward estimates in the potentially rewarding regions of the similarity space. While PCZ borrows the idea of contextual zooming from \cite{slivkins2014contextual}, the arm selection strategy, the regret notion, and the analysis of the regret of PCZ are substantially different from the prior work. \rev{Notably, the regret bound of the contextual zooming algorithm in \cite{slivkins2014contextual} depends on the zooming dimension $d_z$, while our regret bound depends on the Pareto zooming dimension $d_p$. In our case, $d_z$ depends on the objective selected by the contextual zooming algorithm. While, in general, $d_z$ is smaller than $d_p$,\footnote{\rev{Implicitly, both $d_z$ and $d_p$ are functions of a constant $\tilde{c}$ that is related to an $r$-packing of a subset of the similarity space, and the regret bounds hold for any value of $\tilde{c}>0$. Here, we compare $d_z$ and $d_p$ given the same constant.}} it is impossible to guarantee fairness over context-arm pairs in the Pareto front by using the contextual zooming algorithm.} In addition, due to the multi-dimensional nature of the problem, different bounding techniques are required to show that PCZ selects arms that are near the Pareto front with high probability. \rev{Moreover, we make use of a different concentration inequality that allows us to deal with noise processes that are conditionally $1$-sub-Gaussian, and also provide finite-time bounds on the Pareto regret and the expected Pareto regret.}

Learning with multiple objectives is mainly investigated under the classical MAB model without contexts.
Numerous algorithms are proposed to minimize the Pareto regret, which measures the total loss due to selecting arms that are not in the Pareto front. For this problem, an algorithm called ParetoUCB1 is shown to achieve $O(\log T)$ Pareto regret \cite{drugan2013designing}. This algorithm simply uses the UCB indices instead of the expected arm rewards to select an arm from the estimated Pareto front in each round. Other algorithms include the Pareto Thompson sampling \cite{yahyaa2014annealing}, the Annealing Pareto \cite{yahyaa2014annealing} and the Pareto-KG \cite{yahyaa2014knowledge}. 
Another line of research aims to scalarize the multi-objective problem by assigning weights to each objective. For instance, Scalarized UCB1 \cite{drugan2013designing} learns the best scalarization functions among a given set of scalarization functions, and achieves $O(S'\log(T/S'))$ scalarized regret where $S'$ is the number of scalarization functions used by the algorithm. 
In another work \cite{van2014multi}, the hierarchical optimistic optimization strategy for the ${\cal X}$-armed bandit problem \cite{bubeck2011x} is extended to the multi-objective setup, but theoretical regret analysis of the proposed method is not given.

Apart from the above works, a specific multi-objective contextual bandit problem with dominant and non-dominant objectives and finite number of arms is considered in \cite{tekin2017multi}. They define the regret in each objective separately with respect to a benchmark that always picks a specific arm in the Pareto front that favors the reward in the dominant objective over the reward in the non-dominant objective. For this problem, they propose a contextual bandit algorithm that uniformly partitions the context set, and prove that it achieves $\tilde{O} ( T^{(\alpha+d)/(2\alpha+d)} )$ Pareto regret, where $d$ is the dimension of the context and $\alpha$ is a similarity information dependent constant. Our work significantly differs from \cite{tekin2017multi} since we consider a very general similarity space and an adaptive contextual zooming algorithm. Moreover, we do not prioritize the objectives. A detailed comparison of our work with the related works is given in Table \ref{table:related}. 
\section{PROBLEM DESCRIPTION} \label{sec:3}
The system operates in rounds indexed by $t \in \{1, 2, \ldots\}$. At the beginning of each round, the learner observes a context $x_t$ that comes from a $d_x$-dimensional context set ${\cal X}$. 
Then, the learner chooses an arm $y_t$ from a $d_y$-dimensional arm set ${\cal Y}$. After choosing the arm, the learner obtains a $d_r$-dimensional random reward vector $r_t := ( r^1_t, \ldots, r^{\rev{d_r}}_t )$ where $r^i_t$ denotes the reward obtained from objective $i \in \{1, \ldots, \rev{d_r} \}$ in round $t$. Let $\mu^i_{y}(x)$ denote the expected reward of arm $y$ in objective $i$ for context $x$ and $\mu_y(x) := ( \mu^1_{y}(x), \ldots, \mu^{d_r}_{y}(x) )$.  The random reward vector obtained from arm $y_t$ in round $t$ is given as $r_t := \mu_{y_t}(x_t) + \kappa_t$ where $\kappa_t$ is the $d_r$-dimensional noise process whose marginal distribution for each objective is conditionally 1-sub-Gaussian, i.e., $\forall \lambda \in \mathbb{R}$
\begin{align}
\expect{e^{\lambda \kappa^i_t } | y_{1:t}, x_{1:t}, \kappa_{1:t-1}}&\leq \exp{(\lambda^2 / 2)} \notag
\end{align}
where $b_{1:t} := (b_1, \ldots, b_t)$. 
Context and arm sets together constitute the set of feasible context-arm pairs, denoted by ${\cal P} := {\cal X} \times {\cal Y}$. We assume that the Lipschitz condition holds for the set of feasible context-arm pairs with respect to the expected rewards for all objectives.
\begin{assumption} \label{assm:1}
For all $i \in \{1, \ldots , d_r\}$, $y, y' \in {\cal Y}$ and $x,x' \in {\cal X}$, we have
\begin{align}
|\mu^i_{y}(x) - \mu^i_{y'}(x')| \leq D((x, y), (x', y')) \notag
\end{align}
where $D$ is the distance function known by the learner such that $D((x, y), (x', y')) \leq 1$ for all $x,x' \in {\cal X}$ and $y,y' \in {\cal Y}$.
\end{assumption}
$({\cal P}, D)$ denotes the \textit{similarity space}. 
In the multi-objective bandit problem, since the objectives might be conflicting, finding an arm that simultaneously maximizes the expected reward in all objectives is in general not possible. Thus, an intuitive way to define optimality is to use the notion of Pareto optimality.
\begin{defn}[Pareto optimality] \label{def:1}
	(i) An arm $y$ is \emph{weakly dominated} by arm $y'$ given context $x$, denoted by $\mu_{y}(x) \preceq \mu_{y'}(x)$ or $\mu_{y'}(x) \succeq \mu_{y}(x)$, if $\mu^{i}_{y}(x) \leq \mu^{i}_{y'}(x), \forall i \in \{1, \ldots , d_r\}$. \\
	(ii) An arm $y$ is \emph{dominated} by arm $y'$ given context $x$, denoted by $\mu_{y}(x) \prec \mu_{y'}(x)$ or $\mu_{y'}(x) \succ \mu_{y}(x)$, if it is weakly dominated and $\exists i \in \{1, \ldots , d_r\}$ such that $\mu^{i}_{y}(x) < \mu^{i}_{y'}(x)$. \\
	(iii) Two arms $y$ and $y'$ are incomparable given context $x$, denoted by $\mu_{y}(x) || \mu_{y'}(x)$, if neither arm dominates the other. \\
	(iv) An arm is \emph{Pareto optimal} given context $x$ if it is not dominated by any other arm given context $x$. The set of all Pareto optimal arms given a particular context $x$, is called the \emph{Pareto front}, and is denoted by ${\cal O}(x)$.
\end{defn}
The expected loss incurred by the learner in a round due to not choosing an arm in the Pareto front is equal to the Pareto suboptimality gap (PSG) of the chosen arm, which is defined as follows. 
\begin{defn}[PSG]
	The PSG of an arm $y \in {\cal Y}$ given context $x$, denoted by $\Delta_{y}(x)$, is defined as the minimum scalar $\epsilon \geq 0$ that needs to be added to all entries of $\mu_{y}(x)$ such that $y$ becomes a member of the Pareto front. Formally, 
	\begin{align*}
	\Delta_{y}(x) := \inf_{\epsilon \geq 0} \epsilon ~~ \text{s.t.} ~~  (\mu_{y}(x) + \bs{\epsilon}) \> || \> \mu_{y'}(x), \forall y' \in {\cal O}(x) 
	\end{align*}
	where $\bs{\epsilon}$ is a $d_r$-dimensional vector, whose entries are $\epsilon$.
\end{defn}
We evaluate the performance of the learner using the Pareto regret, which is given as 
\begin{align}
\text{Reg}(T) := \sum_{t=1}^{T} \Delta_{y_t}(x_t) . \label{eqn:PRregret} 
\end{align}
The Pareto regret measures the total loss due to playing arms that are not in the Pareto front. 
The goal of the learner is to minimize its Pareto regret while ensuring fairness over the Pareto optimal arms for the observed contexts. Our regret bounds depend on the Pareto zooming dimension, which is defined below.
\begin{defn}\label{defn:zooming}
(i) $\hat{\cal P} \subset {\cal P}$ is called an $r$-packing of ${\cal P}$ if all $z,z' \in \hat{\cal P}$ satisfies $D(z,z') \geq r$. For any $r > 0$, the $r$-packing number of ${\cal P}$ is  $\rev{A_{r}^{packing} ({\cal P})} := \max \{ |\hat{\cal P}| : \hat{\cal P} \text{ is an $r$-packing of } {\cal P} \}$. \\
(ii) For a given $r>0$, let ${\cal P}_{\mu, r} := \{ (x,y) \in {\cal P}:  \Delta_{y} (x) \leq  12 r \}$ denote the set of near-optimal context-arm pairs. The Pareto $r$-zooming number $N_r$ is defined as the $r$-packing number of ${\cal P}_{\mu, r}$. \\
(iii) The Pareto zooming dimension given any constant $\tilde{c} >0$ is defined as $d_p( \tilde{c} ) := \inf \{ d > 0 : N_r \leq \tilde{c} r^{-d},  \forall r \in (0,1)\} $. With an abuse of notation we let $d_p := d_p(p)$ for $p > 0$.	
\end{defn}

\section{THE LEARNING ALGORITHM}

In this section, we present a multi-objective contextual bandit algorithm called \textit{Pareto Contextual Zooming} (PCZ). Pseudo-code of PCZ is given in Algorithm \ref{algorithm:PCZ}.
 \begin{algorithm}[h!]
\caption{Pareto Contextual Zooming}\label{algorithm:PCZ}
\begin{algorithmic}[1] 
\STATE Input: $({\cal P}, D)$, $T$, $\delta$ 
\STATE Data: Collection ${\cal B}$ of "active balls" in (${\cal P}, {\cal D}$); counters $N_{B}$ and estimates $\hat{\mu}_{B}^i$, $\forall B \in {\cal B}$, $\forall i \in \{1,\ldots,d_r\}$
\STATE Init: Create ball $B$, with $r(B) = 1$ and an arbitrary center in ${\cal P}$. $ {\cal B} \leftarrow  \{ B \}$ 
\STATE  $\hat{\mu}_{B}^i = 0$, $\forall i \in \{1,\ldots,d_r\}$ and $N_{B} = 0$
\WHILE{$1 \leq t \leq T$}
\STATE Observe $x_t$
\STATE $\hat{{\cal R}}(x_t) \leftarrow  \{ B \in {\cal B}: (x_t,y) \in \text{dom}_t(B) $ for some $y \in {\cal Y} \}$
\STATE $\hat{{\cal A}}^* \leftarrow  \{  B \in \hat{{\cal R}}(x_t) : {g}_{B} \not \prec {g}_{B'}, \forall B' \in \hat{{\cal R}}(x_t) \}$
\STATE Select an arm $y_t$ uniformly at random from $\{ y: (x_t,y) \in \cup_{B \in \hat{{\cal A}}^*} \text{dom}_t( B) \}$, and observe the rewards $r^i$, $\forall i \in \{1,\ldots,d_r\}$
\STATE Uniformly at random choose a ball $\hat{B} \in \hat{{\cal A}}^*$ whose domain contains $(x_t, y_t)$
\IF{$u_{\hat{B}} \leq r(\hat{B})$}
\STATE Activate \revn{(create)} a new ball $B'$ whose center is $(x_t,y_t)$ and radius is $r(B') = r(\hat{B})/2$
\STATE ${\cal B} \leftarrow {\cal B} \cup B'$, and $\hat{\mu}_{B'}^i =  N_{B'} = 0$, $\forall i \in \{1,\ldots,d_r\}$.
\STATE Update the domains of balls in ${\cal B}$.
\ENDIF
\STATE Update estimates $\hat{\mu}_{\hat{B}}^i = ((\hat{\mu}_{\hat{B}}^i N_{\hat{B}}) + r^i) / ( N_{\hat{B}}+1)$, $\forall i \in \{1,\ldots,d_r\}$ and the counter $N_{\hat{B}} = N_{\hat{B}} + 1$ 
\ENDWHILE
\end{algorithmic}
\end{algorithm}
PCZ is a multi-objective extension of the single objective contextual zooming algorithm \cite{slivkins2014contextual}. The proposed algorithm partitions the similarity space non-uniformly according to the arms selected, and the contexts and rewards observed in the past by using a set of \textit{active balls} ${\cal B}$ \rev{which may change from round to round}. Each active ball $B \in {\cal B}$ has a radius $r(B)$, center $(x_B, y_B)$ and a domain in the similarity space. The domain of ball $B$ at \rev{the beginning of round $t$} is denoted by $\text{dom}_t(B)$, \rvs{and is defined as the subset of $B$ that excludes all active balls at the beginning of round $t$ that have radius strictly smaller than $r(B)$, i.e., $\text{dom}_t(B) := B \setminus (\cup_{B' \in {\cal B}: r(B') < r(B)} B')$.} \rev{The domains of all active balls cover the similarity space}.

Initially, PCZ takes as inputs the time horizon $T$,\footnote{\rev{While PCZ requires $T$ as input, it is straightforward to extend it to work without knowing $T$ beforehand, by using a standard method, called the doubling trick.}} the similarity space $({\cal P}, D)$, the confidence parameter $\delta \in (0,1)$, and creates an active ball centered at a random point in the similarity space with radius $1$ whose domain covers the entire similarity space. At the beginning of round $t$, PCZ observes the context $x_t$, and finds the set of relevant balls denoted by $\hat{{\cal R}}(x_t) :=  \{ B \in {\cal B}: (x_t,y) \in \text{dom}_t(B) $ for some $y \in {\cal Y} \}$. 

After finding the set of relevant balls, 
PCZ uses the principle of \textit{optimism under the face of uncertainty} to select a ball and an arm. In this principle, the estimated rewards of the balls are inflated by a certain level, such that the inflated reward estimates (also called indices) become an upper confidence bound (UCB) for the expected reward with high probability. Then, PCZ selects a ball whose index is in the Pareto front.
In general, this allows the balls that are rarely selected to get explored (because their indices will remain high due to the sample uncertainty being high), which enables the learner to discover new balls that are potentially better than the frequently selected balls in terms of the rewards.

The index for each relevant ball is calculated in the following way:
First, PCZ computes a pre-index for each ball in ${\cal B}$ given by
\begin{align}
g^{i,pre}_{B} :=& \hat{\mu}_{B}^i + u_{B} + r(B), ~ i \in  \{1,\ldots,d_r\} \notag 
\end{align}
which is the sum of the sample mean reward $\hat{\mu}_{B}^i$, the sample uncertainty $u_{B} := \sqrt{2 A_{B} /N_{B}}$, where \rev{$A_{B} := (1+2\log(2 \sqrt{2} d_r T^{\frac{3}{2}}/ \delta ))$} and $N_{B}$ is the number of times ball $B$ is chosen, and the contextual uncertainty $r(B)$. The sample uncertainty represents the uncertainty in the sample mean estimate of the expected reward due to the limited number of random samples in ball $B$ that are used to form this estimate. On the other hand, the contextual uncertainty represents the uncertainty in the sample mean reward due to the dissimilarity of the contexts that lie within ball $B$. When a ball is selected, its sample uncertainty decreases but its contextual uncertainty is always fixed.
\rev{Essentially, the pre-index $g^{i,pre}_{B}$ is a UCB for the expected reward at the center of ball $B$ in objective $i$.}
PCZ uses these pre-indices to compute an index for each relevant ball, given as
\begin{align}
{g}_B &:= (g^1_B, \ldots, g^{d_r}_B) \text{ where } \notag \\
g^i_B :=& r(B) + \min_{B' \in {\cal B} } (g^{i,pre}_{B'} + D(B', B))  \notag 
\end{align}
and $D(B',B)$ represents the distance between the centers of the balls $B'$ and $B$ in the similarity space. 

\rev{After the indices of the relevant balls are calculated, PCZ computes the Pareto front among the set of balls in $\hat{{\cal R}}(x_t)$ by using $g_B$ for ball $B \in \hat{{\cal R}}(x_t)$ as a proxy for the expected reward (see Definition \ref{def:1}), which is given as $\hat{{\cal A}}^* := \{  B \in \hat{{\cal R}}(x_t) : {g}_{B} \not \prec {g}_{B'}, \forall B' \in \hat{{\cal R}}(x_t) \}$, and uniformly at random selects an arm in the union of the domains of the balls whose indices are in the Pareto front. After observing the reward of the selected arm, it uniformly at random picks a ball whose index is in the Pareto front and contains $(x_t, y_t)$, updates its parameters, and the above procedure repeats in the next round. This selection ensures fairness over the estimated set of Pareto optimal arms in each round.}

The remaining important issue is how to create the balls in order to trade-off sample uncertainty and contextual uncertainty in an optimal way. This is a difficult problem due to the fact that the learner does not know how contexts arrive beforehand, and the rate that the contexts arrive in different regions of the context set can dynamically change over time. To overcome these issues, PCZ uses the concept of contextual zooming. Basically, PCZ adaptively partitions the context space according to the past context arrivals, arm selections and obtained rewards. Specifically, the radius of the balls are adjusted to optimize the two sources of error described above. For this, when the sample uncertainty of \revn{the selected ball} $B$ \revn{is found to be smaller} than or equal to the radius of the ball (\revn{say in round $t$}), a new child ball $B'$ centered at $(x_t,y_t)$ whose radius is equal to the half of the parent ball's (ball $B$'s) radius is created, and the domains of all active balls are updated.\footnote{We also use $B^{\text{par}}$ to denote the parent of a ball $B$.} 
Note that when a child ball is created, it does not contain any sample, so its sample uncertainty is infinite. This results in an infinite pre-index. For this reason, $g^i_B$ is used instead of $g^{i,pre}_{B}$ as an enhanced UCB, which is in general tighter than the UCB based on the pre-index.

	\section{REGRET ANALYSIS} 
In this section, we prove that PCZ achieves (i) $\tilde  O(T^{(1 + d_p) / (2+ d_p)})$ Pareto regret with probability at least $1-\delta$, and (ii) $\tilde  O(T^{(1 + d_p) / (2+ d_p)})$ expected Pareto regret, where $d_p$ is the Pareto zooming dimension.

First, we define the variables that will be used in the Pareto regret analysis. 
For an event ${\cal F}$, let ${\cal F}^c$ denote the complement of that event.
For all the parameters defined in the pseudo-code of PCZ, we explicitly use the round index $t$, when referring to the value of that parameter at the beginning of round $t$. For instance, $N_{B}(t)$ denotes the value of $N_{B}$ at the beginning of round $t$, and ${\cal B}(t)$ denotes the set of balls created by PCZ by the beginning of round $t$. Let ${\cal B}'(T)$ denote the set of balls chosen at least once by the end of round $T$. Note that ${\cal B}'(T) \subset {\cal B}(T)$ and ${\cal B}(T)$ is a random variable that depends both on the contexts arrivals, selected arms and observed rewards.
\revn{Let $R^i_{B}(t)$ denote the random reward of ball $B$ in objective $i$ at round $t$ and let $\tau_B(t)$ denote the first round after the round in which ball $B$ is chosen by PCZ for the $t$th time.}
\revn{Moreover, the round that comes just after $B \in {\cal B}(T)$ is created is denoted by $\tau_B(0)$, and the domain of $B$ when it is created is denoted by dom($B$). Hence, $\text{dom}_t(B) \subseteq \text{dom}(B)$, $\forall t \in \{\rev{\tau_B(0)}, \ldots, T\}$. For $B \in {\cal B}'(T)$, let ${\cal T}_B$ denote the set of rounds in $\{\tau_B(0), \ldots, T\}$ in which ball $B$ is selected by PCZ.}
Also let $\tilde{x}_{B}(t) := x_{\tau_B(t)-1}$,
$\tilde{y}_{B}(t) := y_{\tau_B(t)-1}$, 
$\tilde{R}^i_{B}(t) := R^i_{B}(\tau_B(t)-1)$, 
$\tilde{\kappa}^i_{B}(t) := \kappa^i_{\tau_B(t)-1}$,
$\tilde{N}_{B}(t) := N_{B}(\tau_B(t))$, 
$\tilde{\mu}^i_{B}(t) := \hat{\mu}^i_{B}(\tau_B(t))$, 
$\tilde{g}^{i,pre}_{B}(t) := g^{i,pre}_{B} (\tau_B(t))$, 
$\tilde{g}^{i}_{B}(t) := g^{i}_{B} (\tau_B(t))$, 
 and $\tilde{u}_{B}(t) := u_{B}(\tau_B(t))$.
We note that all inequalities that involve random variables hold with probability one unless otherwise stated.

Let
\begin{align}
\text{Reg}_B(T) := \sum_{t=1}^{\rev{N_B(T+1)}} \Delta_{\tilde{y}_{B}(t)}(\tilde{x}_{B}(t)) \notag 
\end{align}
denote the Pareto regret incurred in ball $B\in {\cal B}'(T)$ for rounds in ${\cal T}_B$. Then, the Pareto regret \revn{in \eqref{eqn:PRregret}} can be written as
\begin{align}
\text{Reg}(T) = \sum_{\rev{B \in {\cal B}'(T)}} \text{Reg}_B(T)  .\notag 
\end{align}
Next, we define the following lower and upper bounds:
${L}^i_{B}(t) := \hat{\mu}^i_{B}(t) - u_{B}(t)$, ${U}^i_{B}(t) := \hat{\mu}^i_{B}(t) + u_{B}(t)$,
$\tilde{L}^i_{B}(t) := \tilde{\mu}^i_{B}(t) - \tilde{u}_{B}(t)$ and $\tilde{U}^i_{B}(t) := \tilde{\mu}^i_{B}(t) + \tilde{u}_{B}(t)$ for $i \in \{1, \ldots, d_r\}$. 
Let 
\begin{align*}
{\text{UC}}^i_{B} := 
\bigcup_{t = \tau_B(0)}^{T+1} \{ \mu^i_{{y}_{B}} &({x}_{B}) \notin \notag \\
&[ {L}^i_{B}(t) - r(B) , {U}^i_{B}(t) + r(B) ] \} \notag
\end{align*}
denote the event that the learner is not confident about its reward estimate in objective $i$ in ball $B$ for at least once from round $\tau_B(0)$ to round $T$,
and
\begin{align*}
\tilde{\text{UC}}^i_{B} := 
\bigcup_{\rvs{t=0}}^{\rev{N_{B}(T+1)}} \{ \mu^i_{{y}_{B}}&({x}_{B}) \notin \notag \\
&[ \tilde{L}^i_{B}(t) - r(B) , \tilde{U}^i_{B}(t) + r(B) ] \} . \notag
\end{align*}
Let $\tau_B( N_B(T+1) + 1) = T+2$. Then, for $z=0,\ldots,N_B(T+1)$ the events $\{  \mu^i_{{y}_{B}} ({x}_{B}) \notin
[ {L}^i_{B}(t) - r(B) , {U}^i_{B}(t) + r(B) ] \}$ and $\{  \mu^i_{{y}_{B}} ({x}_{B}) \notin 
[ {L}^i_{B}(t') - r(B) , {U}^i_{B}(t') + r(B) ] \}$ are identical for any $t,t' \in \{\tau_B(z),\ldots,\tau_B(z+1) - 1\}$. Thus, 
\begin{align*}
& \bigcup_{t = \tau_B(z)}^{\tau_B(z+1) - 1}    \{  \mu^i_{{y}_{B}} ({x}_{B}) \notin 
[ {L}^i_{B}(t) - r(B) , {U}^i_{B}(t) + r(B) ] \} \notag \\
&= \{ \mu^i_{{y}_{B}}({x}_{B}) \notin 
[ \tilde{L}^i_{B}(z) - r(B) , \tilde{U}^i_{B}(z) + r(B) ] \} .
\end{align*}
Hence, we conclude that ${\text{UC}}^i_{B} = \tilde{\text{UC}}^i_{B}$. 
Let $\text{UC}_B := \cup_{i \in \{1, \ldots, d_r\}} \text{UC}^i_{B}$ and $\text{UC} := \cup_{B \in \rev{{\cal B}(T)}} \text{UC}_B$. Next, we will bound $\expect{ \text{Reg}(T) }$. 
We have
\begin{align}
\expect{ \text{Reg}(T) } &= \expect{ \text{Reg}(T) | \text{UC} } 
\Pr ( \text{UC}  )  \notag \\
&+ \expect{ \text{Reg}(T) | \text{UC}^c } \Pr ( \text{UC}^c )  \notag \\
&\leq C_{\max} T \Pr ( \text{UC}  )  
+ \expect{ \text{Reg}(T) | \text{UC}^c }  \label{eqn:partitiondecompose2}
\end{align}
where $C_{\max} := \sup_{ (x,y) \in {\cal P} } \Delta_y(x)$.
Since ${\cal B}(T)$ is a random variable, we have
\begin{align}
\Pr( \text{UC} ) = \int  \Pr( \text{UC} |  {\cal B}(T))  d Q( {\cal B}(T) ) \label{Brand1}
\end{align}
where $Q( {\cal B}(T) )$ denotes the distribution of ${\cal B}(T)$. We also have
\begin{align}
\Pr( \text{UC} |  {\cal B}(T)) & = \Pr( \cup_{B \in {\cal B}(T)} \text{UC}_B |  {\cal B}(T)) \notag \\
& \leq \sum_{B \in {\cal B}(T) } \Pr( \text{UC}_B | {\cal B}(T) ) \label{Brand2}.
\end{align}
We proceed by bounding the term $\Pr( \text{UC} )$ in \eqref{eqn:partitiondecompose2}. For this, first we bound $\Pr ( \text{UC}_B | {\cal B}(T) )$ in the next lemma, whose proof is given in the supplemental document. 

\begin{lemma} \label{lemma:prUC}
When PCZ is run, we have
$\Pr( \text{UC}_B | {\cal B}(T) ) \leq \delta / T$, $\forall B \in {\cal B}(T)$.   
\end{lemma}

Next, we bound $\Pr( \text{UC} )$. Since $|{\cal B}(T)| \leq T$, by using union bound over all created balls, \eqref{Brand1} and \eqref{Brand2} we obtain
\begin{align*}
\Pr( \text{UC} ) &\leq  \int \Pr( \text{UC} |  {\cal B}(T))  d Q( {\cal B}(T) ) \notag \\
&\leq \int \bigg( \sum_{B \in {\cal B}(T) } \Pr( \text{UC}_B | {\cal B}(T) ) \bigg)  d Q( {\cal B}(T) ) \notag \\
&\leq \frac{\delta T}{T} \int d Q( {\cal B}(T) ) \notag \leq \delta \notag .
\end{align*}
Then, by using \eqref{eqn:partitiondecompose2},
\begin{align}
\expect{ \text{Reg}(T) } &\leq C_{\max} T \Pr ( \text{UC}  )  
+ \expect{ \text{Reg}(T) | \text{UC}^c } \notag \\
&\leq  C_{\max} T \delta + \expect {\text{Reg}(T)  | \text{UC}^c} . \label{eq:expparreg}
\end{align}
In the remainder of the analysis we bound $\text{Reg}(T)$ on event $\text{UC}^c$. Then, we conclude the analysis by using this to bound \eqref{eq:expparreg}.
For simplicity of notation, in the following lemmas we use $B(t)$ to denote the ball selected in round $t$.
\begin{lemma} \label{lemma:paretoOptimal}
Consider a virtual arm with expected reward vector ${g}_{B(t)}(t)$, where $g^i_{B(t)}(t)$ denotes the expected reward in objective $i$. On event $\text{UC}^c$, we have
\begin{align}
{g}_{B(t)}(t) \not \prec {\mu}_{y}( x_t), \forall y \in {\cal Y} . \notag
\end{align}
\end{lemma}
\begin{proof}
For any relevant ball $B \in \hat{{\cal R}}(x_t)$ and $i \in \{1, \ldots, d_r\}$, let $\revn{\tilde{B}_i} := \argmin_{B' \in {\cal B}(t)} \{ g^{i,pre}_{B'}(t) + D(B,B') \}$. Then, $\forall y$ such that $(x_t,y) \in \text{dom}_t(B)$, we have
\begin{align}
g^i_{B}(t) &= r(B) + g^{i,pre}_{\revn{\tilde{B}_i}}(t) + D(B, \revn{\tilde{B}_i}) \notag \\
&= r(B) + \hat{\mu}^i_{\revn{\tilde{B}_i}} (t) +  {u}_{\revn{\tilde{B}_i}}(t) + r (\revn{\tilde{B}_i}) + D(B, \revn{\tilde{B}_i})  \notag \\
&= r(B) + {U}^i_{\revn{\tilde{B}_i}}(t) + r (\revn{\tilde{B}_i}) + D(B, \revn{\tilde{B}_i})  \notag \\
&\geq r(B) + \mu^i_{y_{\revn{\tilde{B}_i}}}(x_{\revn{\tilde{B}_i}})+ D(B, \revn{\tilde{B}_i})  \notag \\
&\geq r(B) + \mu^i_{y_{B}}( x_{B} ) \geq \mu^i_{y}(x_t)  \notag
\end{align}
where the first inequality holds by the definition of $\text{UC}^c$, and the second and the third inequalities hold by Assumption \ref{assm:1}.
According to the above inequality ${g}_{B}(t) \succeq {\mu}_{y}( x_t)$,  $\forall y$ such that $(x_t, y) \in \text{dom}_\rev{t}(B)$. \revn{We also know that $\forall y \in {\cal Y}$, $\exists B \in \hat{{\cal R}}(x_t)$ such that $(x_t, y) \in \text{dom}_\rev{t}(B)$.}
Moreover, by the selection rule of PCZ, ${g}_{B(t)}(t) \not \prec{g}_{B}(t) $ for all $B \in \hat{{\cal R}}(x_t)$. By combining these results, ${g}_{B(t)}(t) \not \prec {\mu}_{y}(x_t), \forall y \in {\cal Y} $, and hence, the virtual arm with expected reward vector ${g}_{B(t)}(t)$ is not dominated by any of the arms.
\com{Since $g_B(t)$ weakly dominates $\mu_y(x_t)$ for $y$ in its domain, and $g_{B(t)}(t)$ is not dominated by $g_B(t)$, we have $g_{B(t)}(t)$ is not dominated by $\mu_y(x_t)$.}
\end{proof}
\begin{lemma} \label{lemma:indexdiff}
When PCZ is run, on event $\text{UC}^c$, we have
\begin{align}
 \Delta_{y_t}(x_t) \leq  14 r(B(t)) ~~ \forall t \in \{1,\ldots,T\} . \notag
\end{align}
\end{lemma} 
\begin{proof}
This proof is similar to the proof in \cite{slivkins2014contextual}.
To bound the PSG of the selected ball, we first bound the index of the selected ball $g^i_{B(t)}(t)$.
Recall that $B^{par}(t)$ denotes the parent ball of the selected ball $B(t)$.\footnote{\rev{The bound for $B(1)$ is trivial since it contains the entire similarity space.}} We have
\begin{align}
g^{i,pre}_{B^{par}(t)}(t) &= \hat{\mu}^i_{B^{par}(t)} (t) + r (B^{par}(t)) + {u}_{B^{par}(t)}(t)\notag \\
&= {L}^i_{B^{par}(t)}(t) + r (B^{par}(t)) + 2 {u}_{B^{par}(t)}(t)\notag \\
&\leq \mu^i_{y_{B^{par}(t)}}(x_{B^{par}(t)} ) \notag \\
&+ 2 r (B^{par}(t)) +2  {u}_{B^{par}(t)}(t) \notag \\
&\leq  \mu^i_{y_{B^{par}(t)}}(x_{B^{par}(t)} ) + 4 r (B^{par}(t)) \notag \\ 
&\leq  \mu^i_{y_{B(t)}}(x_{B(t)})  + 5 r (B^{par}(t)) \label{eqn:parpre}
\end{align}
where the first inequality holds by the definition of $\text{UC}^c$, the second inequality holds since ${u}_{B^{par}(t)}(t) \leq r (B^{par}(t))$, and the third inequality holds due to Assumption \ref{assm:1}. We also have
\begin{align}
g^i_{B(t)}(t) &\leq r(B(t)) + g^{i,pre}_{B^{par}(t)}(t) + D(B^{par}(t),B(t)) \notag \\
& \leq r({B}(t)) + g^{i,pre}_{B^{par}(t)}(t) + r(B^{par}(t)) \notag \\
& \leq r({B}(t)) +  \mu^i_{y_{B(t)}}(x_{B(t)})  + 6 r (B^{par}(t))  \notag \\
& \leq  \mu^i_{y_{B(t)}}(x_{B(t)}) + 13 r ({B}(t)) \notag \\
& \leq  \mu^i_{y_t}( x_t ) + 14 r ({B}(t))  \notag 
\end{align}
where the third inequality follows from \eqref{eqn:parpre}.
Since $g^i_{{B}(t)}(t) - \mu^i_{y_t}( x_t )  \leq  14 r ({B}(t))$ for all $i \in \{1,\ldots, d_r\}$ and the virtual arm is not dominated by any arm in the Pareto front by Lemma \ref{lemma:paretoOptimal}, the PSG of the selected arm is bounded by $\Delta_{y_t}(x_t) \leq  14 r(B(t))$.
\com{If we add $14r(B(t))$ to each $\mu^i_{y_t}(x_t)$, then $\mu_{y_t}(x_t)$ will weakly dominate $g_{B(t)}(t)$. From Lemma 2 we know that $g_{B(t)}(t)$ is not dominated by any arm $y$ with expected reward $\mu_y(x_t)$. This means that $\mu_{y_t}(x_t) + 14r(B(t))$ is not dominated by any arm $\mu_y(x_t)$, and hence, it is in the Pareto front. By Definition 2, the PSG is at most $14r(B(t))$.}
\end{proof}
\begin{lemma}\label{lemma:nr}
	When PCZ is run, on event $UC^c$, the maximum number of radius $r$ balls that are created by round $T$ is bounded by the Pareto $r$-zooming number $N_r$ given in \rev{Definition \ref{defn:zooming}}. \revn{Moreover, in any round $t$ in which a radius $r$ ball is created, we have $\Delta_{y_t}(x_t) \leq 12 r$.}
\end{lemma}
\begin{proof}
	Assume that a new ball is created at round $t$ whose parent is $B(t)$. Let $B'(t)$ denote the created ball. We have,	
	\begin{align}
		g^i_{B(t)}(t) &\leq r(B(t)) + g^{i,pre}_{B(t)}(t) \notag \\
		&= \hat{\mu}^i_{B(t)} (t) + 2 r (B(t)) + {u}_{B(t)}(t)\notag \\
		&= {L}^i_{B(t)}(t) + 2 r (B(t)) + 2 {u}_{B(t)}(t) \notag \\
		&\leq \mu^i_{y_{B(t)}}(x_{B(t)} ) + 3 r (B(t)) + 2 {u}_{B(t)}(t) \notag \\
		&\leq \mu^i_{y_{B(t)}}(x_{B(t)} ) + 5 r (B(t)) \notag \\
		&\leq \mu^i_{y_{B'(t)}}(x_{B'(t)} ) + 6 r (B(t)) \notag	\\
		&\leq \mu^i_{y_{B'(t)}}(x_{B'(t)} ) + 12 r (B'(t)) \notag
	\end{align}  
	where the first inequality follows from the definition of $g^i_{B(t)}(t)$, the second inequality holds by the definition of $\text{UC}^c$, the third inequality holds due to the fact that $B(t)$ is a parent ball, the fourth inequality holds due to Assumption \ref{assm:1}, and the last inequality follows from $r(B'(t)) = r(B(t))/2$.
	Similar to the proof of Lemma \ref{lemma:indexdiff}, since $g^i_{{B}(t)}(t) - \mu^i_{y_{B'(t)}}(x_{B'(t)} )  \leq  12 r ({B}'(t))$ for all $i \in \{1,\ldots, d_r\}$ and the virtual arm is not dominated by any arm in the Pareto front by Lemma \ref{lemma:paretoOptimal}, the PSG of point \revn{$(x_{B'(t)} ,y_{B'(t)}) = (x_t, y_t)$} is bounded by $12 r ({B}'(t))$. This implies that center of the ball created in any round $t$ has a PSG that is at most $12 r ({B}'(t))$. Thus, the center of $B'(t)$ is in ${\cal P}_{\mu, r ({B}'(t))}$.
	Next, consider any two balls $B$ and $B'$ with radius $r$ created by PCZ. Based on the ball creation and domain update rules of PCZ, the distance between the centers of these balls must be at least $r$. As a result, the maximum number of radius $r$ balls created is bounded by the $r$-packing number of ${\cal P}_{\mu, r}$, which is $N_r$.
\end{proof}

The following lemma (proof is given in the supplemental document), bounds the regret of PCZ in terms of $N_r$ by using the results in Lemmas 3 and 4.
\begin{lemma} \label{lemma:pregt}
On event $\text{UC}^c$, the Pareto regret of PCZ by round $T$ is bounded by $ \text{Reg}(T) \leq 28 T r_0 + \sum_{r = 2^{-i}: i \in \mathbb{N}, r_0 \leq r \leq 1} 56 r^{-1} N_r \log (  2 \sqrt{2} d_r T^{\frac{3}{2}} e / \delta) $ for any $r_0 \in (0,1]$.
\end{lemma}
The following theorem gives a high probability Pareto regret bound for PCZ.
\begin{theorem} \label{theorem:1}
For any $p>0$, the Pareto regret of PCZ by round $T$ is bounded \rev{with probability at least $1-\delta$} (on event $\text{UC}^c$) by
\begin{align}
\text{Reg}(T)  &\leq ( 28 + 112 p  \log (  2 \sqrt{2} d_r T^{\frac{3}{2}} e / \delta) ) T^{\frac{1+d_p}{2+d_p}} \notag
\end{align}
where $d_p$ is given in Definition \ref{defn:zooming}.
\end{theorem}
\begin{proof}
Using Definition \ref{defn:zooming} and the result of Lemma \ref{lemma:pregt} on event $\text{UC}^c$, we have
\begin{align}
&\text{Reg}(T) \leq 28 T r_0  \notag \\
&+ \sum_{r = 2^{-i}: i \in \mathbb{N}}^{r_0 \leq r \leq 1} 56 r^{-1} p r^{-d_p} \log (  2 \sqrt{2} d_r T^{\frac{3}{2}} e / \delta)  \notag \\
&\leq  28 T r_0 
 +  56  p \log (  2 \sqrt{2} d_r T^{\frac{3}{2}} e / \delta)  \sum_{r = 2^{-i}: i \in \mathbb{N}}^{r_0 \leq r \leq 1}  r^{-d_p - 1} . \notag
\end{align}
We obtain the final result by setting $r_0 = T^{\frac{-1}{2+d_p}}$.
\end{proof}
\begin{corollary}
When PCZ is run with $\delta = 1/T$, then for any $p>0$, the expected Pareto regret of PCZ by round $T$ is bounded by
\begin{align}
\expect{\text{Reg}(T)} \leq   
& (28 + 112 p  \log ( 2 \sqrt{2} d_r T^{\frac{\rev{5}}{2}} e) ) T^{\frac{1+d_p}{2+d_p}} 
+ C_{\max} .
 \notag
\end{align}
\end{corollary}
	\begin{proof}
Theorem \ref{theorem:1} is used in \eqref{eq:expparreg} with $\delta = 1 /T$.
	\end{proof}
	\section{LOWER BOUND}
	It is shown in \cite{slivkins2014contextual} that for the contextual bandit problem with similarity information the regret lower bound is $\Omega (T^{(1+d_z)/(2+d_z)})$, where $d_z$ is the contextual zooming dimension. 
    We use this to give a lower bound on the Pareto regret. Consider an instance of the multi-objective contextual bandit problem where $\mu^i_y(x) = \mu^j_y(x)$, $\forall (x,y) \in {\cal P}$ and $\forall i,j \in \{1,\ldots,d_r\}$, and $\kappa^i_t = \kappa^j_t$, $\forall t \in \{1,\ldots,T\}$ and $\forall i,j \in \{1,\ldots,d_r\}$. In this case, the contextual zooming dimension of all objectives are equal (i.e., all $d_z$s are equal). Moreover, by definition of the Pareto zooming dimension $d_p = d_z$.  Therefore, this case is equivalent to the single objective contextual bandit problem. Hence, our regret bound becomes $\tilde O (T^{(1+d_z)/(2+d_z)})$ which matches with the lower bound up to a logarithmic factor. 
    
	\section{CONCLUSION}
	In this paper we propose the multi-objective contextual bandit problem which involves multiple and possibly conflicting objectives. Algorithms designed to deal with a single objective can be highly unfair in terms of their rewards in the other objectives. To overcome this issue, we present the Pareto Contextual Zooming (PCZ) algorithm which achieves $\tilde O (T^{(1+d_p)/(2+d_p)})$ Pareto regret where $d_p$ is the Pareto zooming dimension. PCZ randomly alternates between the arms in the estimated Pareto front and ensures the arms in this set are fairly selected. Future work will focus on evaluating the Pareto regret and the fairness of the proposed algorithm in real-world datasets. 
	
	\subsubsection*{Acknowledgments}
	
	This material is based upon work supported by the Scientific and Technological Research Council of Turkey (TUBITAK) under 3501 Program Grant No. 116E229.

	
	\bibliographystyle{ieeetr}
	\bibliography{references}

	\section*{APPENDIX (SUPPLEMENTAL DOCUMENT)}
	\subsection*{APPENDIX A - A CONCENTRATION INEQUALITY \cite{abbasi2011improved,russo2014learning}} \label{app:concentration} 
	Consider a ball $B$ for which the rewards of objective $i$ are generated by a process $\{ R^i_{B}(t) \}_{t=1}^T$ with mean $\mu^i_{B}= \mr{E} [R^i_{B}(t)]$, where the noise $R^i_{B}(t) - \mu^i_{B}$ is 1-sub-Gaussian. Recall that $B(t)$ is the ball selected in round $t$ and $N_B(T)$ is the number of times ball $B$ is selected by the beginning of round $T$. Let $\hat{\mu}^i_B(T) = \sum_{t=1}^{T-1} \mr{I} (B(t) =B ) R^i_{B}(t) / N_B(T)$ for $N_B(T) >0$ and $\hat{\mu}^i_B(T) = 0$ for $N_B(T) = 0$. Then, for any $0 < \theta < 1$ with probability at least $1-\theta$ we have
	\begin{align}
	&\left| \hat{\mu}^i_{B}(T)  - \mu^i_B \right| \notag \\
	& \leq \sqrt{  \frac{2}{N_B(T)} 
		\left(       
		1 + 2 \log \left(  \frac{ (1 + N_B(T) )^{1/2} } {\theta}    \right)  
		\right)  }  ~~ \forall T \in \mathbb{N}.   \notag
	\end{align}

	\subsection*{APPENDIX B - PROOF OF LEMMA 1}

	From the definitions of $\tilde{L}^i_{B}(t)$, $\tilde{U}^i_{B}(t)$ and $\tilde{\text{UC}}^i_{B}$, it can be observed that the event $\text{UC}^i_{B}$ happens when $\tilde{\mu}^i_{B}(t)$ remains away from \rev{$\mu^i_{y_{B}}(x_{B})$ for some $t \in \{0,\ldots,\rev{N_{B}(T+1)} \}$.} Using this information, we can use the concentration inequality given in Appendix A.
	In this formulation expected rewards of the arms must be equal in all time steps, but in our case, $\mu^i_{\tilde{y}_{B}(t)}(\tilde{x}_{B}(t))$ changes since the elements of $\{ \tilde{x}_{B}(t), \tilde{y}_{B}(t) \}_{t=1}^{\rev{N_{B}(T+1)}}$ are not identical which makes distributions of $\tilde{R}^i_{B}(t) $, $t \in \{1, \ldots, \rev{N_{B}(T+1)} \}$ different.
	
	\rev{In order to overcome this issue, we use the sandwich technique proposed in \cite{tekin2017adaptive} and later used in \cite{tekin2017multi}.} \rvs{For any ball $B \in {\cal B}(T)$, we have $\Pr(\mu^i_{{y}_{B}}({x}_{B}) \notin [ \tilde{L}^i_{B}(0) - r(B) , \tilde{U}^i_{B}(0) + r(B) ]) = 0 $ since $\tilde{\mu}^i_{B}(0) = 0$, $\tilde{L}^i_{B}(0) = - \infty $ and $\tilde{U}^i_{B}(0) = \infty$. Thus, for $B \in {\cal B}(T) \setminus {\cal B}'(T)$, we have $\Pr( \text{UC}_B | {\cal B}(T) ) = 0$. Hence, we proceed by bounding the probabilities of the events $\{\mu^i_{{y}_{B}}({x}_{B}) \notin [ \tilde{L}^i_{B}(t) - r(B) , \tilde{U}^i_{B}(t) + r(B) ]\}$, for $t > 0$ and for the balls in ${\cal B}'(T)$.} Recall that
	$\tilde{R}^i_{B}(t) = \mu^i_{\tilde{y}_{B}(t)} ( \tilde{x}_{B}(t))  + \tilde{\kappa}^i_{B}(t)$ 
	and
	$\tilde{\mu}^i_{B}(t) 
	=  \sum_{l=1}^{{\rev{t}}}  \tilde{R}^i_{B}(l)  / {\rvs{t}}$ \rvs{(for $t > 0$ and $B \in {\cal B}'(T)$).}     
	For each $i \in \{1, \ldots, d_r\}$, $B \in \rev{{\cal B}'(T)}$, let 
	\begin{align*}
	\overline{\mu}^i_{B} &= \sup_{(x,y) \in \text{dom}(B)} \mu^i_y(x) ~\text{ and }~
	\underline{\mu}^i_{B} = \inf_{(x,y) \in \text{dom}(B)} \mu^i_y(x) .
	\end{align*}
	We define two new sequences of random variables, whose sample mean values will lower and upper bound $\tilde{\mu}^i_{B}(t)$. The {\em best sequence} is defined as 
	$\{  \bar{R}^i_{B}(t) \}_{t=1}^{ \rev{N_{B}(T+1)} }$ where
	$\overline{R}^i_{B}(t) :=   \overline{\mu}^i_{B} + \tilde{\kappa}^i_{B}(t) $,
	and the {\em worst sequence} is defined as 
	$\{  \underline{R}^i_{B}(t) \}_{t=1}^{ \rev{N_{B}(T+1)} }$ where 
	$\underline{R}^i_{B}(t) := \underline{\mu}^i_{B} + \tilde{\kappa}^i_{B}(t) $.   
	Let 
	$\overline{\mu}^i_{B}(t) := \sum_{l=1}^{\rev{t}} \overline{R}^i_{B}(l)  / \rev{t}$ and 
	$\underline{\mu}^i_{B}(t) := \sum_{l=1}^{\rev{t}} \underline{R}^i_{B}(l) / \rev{t}$. 
	We have
	\begin{align}
	\underline{\mu}^i_{B}(t) \leq \tilde{\mu}^i_{B}(t) \leq  \overline{\mu}^i_{B}(t)  
	~~\forall t \in \rvs{\{1,\ldots,N_{B}(T+1) \}} . \notag
	\end{align}
	Let
	\begin{align}
	\overline{L}^i_{B}(t) &:=  \overline{\mu}^i_{B}(t) - \tilde{u}_{B}(t)       \notag \\
	\overline{U}^i_{B}(t) &:=  \overline{\mu}^i_{B}(t) + \tilde{u}_{B}(t)     \notag \\
	\underline{L}^i_{B}(t) &:=   \underline{\mu}^i_{B}(t) - \tilde{u}_{B}(t)      \notag \\
	\underline{U}^i_{B}(t) &:=  \underline{\mu}^i_{B}(t) + \tilde{u}_{B}(t)    .  \notag
	\end{align}
	It can be shown that 
	\begin{align}
	& \{ \mu^i_{{y}_{B}}( {x}_{B} ) \notin [\tilde{L}^i_{B}(t)- r(B), \tilde{U}^i_{B}(t) + r(B)]  \} \label{eqn:unionbound1} \\
	& \subset \{ \mu^i_{{y}_{B}}( {x}_{B} ) 
	\notin [\overline{L}^i_{B}(t)  - r(B), 
	\overline{U}^i_{B}(t)  + r(B)]  \}  \notag \\
	&\cup  \{ \mu^i_{{y}_{B}}( {x}_{B} ) 
	\notin [\underline{L}^i_{B}(t)   - r(B), 
	\underline{U}^i_{B}(t)  + r(B)]  \} . \notag
	\end{align}
	
	Moreover, the following inequalities can be obtained from Assumption \rev{1}:
	\begin{align}
	& \mu^i_{{y}_{B}}( {x}_{B} )  \leq \overline{\mu}^i_{B}
	\leq \mu^i_{{y}_{B}}( {x}_{B} ) + \rev{r(B)}  \label{eqn:bestbound} \\
	& \mu^i_{{y}_{B}}( {x}_{B} )  - \rev{r(B)}   \leq \underline{\mu}^i_{B}
	\leq \mu^i_{{y}_{B}}( {x}_{B} ) . \label{eqn:worstbound}  
	\end{align}
	
	Using \eqref{eqn:bestbound} and \eqref{eqn:worstbound} it can be shown that
	\begin{align}
	\{ \mu^i_{{y}_{B}}( {x}_{B} ) &\notin [\overline{L}^i_{B}(t)  - r(B), \overline{U}^i_{B}(t)    + r(B)]  \}   \notag \\ 
	&\subset  \{ \overline{\mu}^i_{B} \notin [\overline{L}^i_{B}(t)  , \overline{U}^i_{B}(t) ]  \}, \notag \\
	\{ \mu^i_{{y}_{B}}( {x}_{B} ) &\notin [\underline{L}^i_{B}(t) - r(B), \underline{U}^i_{B}(t)  + r(B)]  \}  \notag \\ 
	&\subset  \{ \underline{\mu}^i_{B} \notin [\underline{L}^i_{B}(t) , \underline{U}^i_{B}(t) ]  \}  .  \notag
	\end{align}
	
	Plugging this to \eqref{eqn:unionbound1}, we get
	\begin{align*}
	& \{ \mu^i_{{y}_{B}}( {x}_{B} ) \notin [\tilde{L}^i_{B}(t)- r(B), \tilde{U}^i_{B}(t) + r(B)]  \}  \\
	&\subset 
	\{ \overline{\mu}^i_{B} \notin [\overline{L}^i_{B}(t)  , \overline{U}^i_{B}(t) ]  \}
	\bigcup
	\{ \underline{\mu}^i_{B} \notin [\underline{L}^i_{B}(t) , \underline{U}^i_{B}(t) ]  \} . 
	\end{align*}
	
	Using the equation above and the union bound we obtain
	\begin{align}
	\Pr( \text{UC}^i_{B} | {\cal B}(T) ) 
	&\leq \Pr \left( \bigcup_{\rvs{t=1}}^{ \rev{N_{B}(T+1)} } \{ \overline{\mu}^i_{B}  \notin [\overline{L}^i_{B}(t)  , \overline{U}^i_{B}(t) ]  \}  \right)   \notag \\
	& \hspace{-0.1in} + \Pr \left( \bigcup_{\rvs{t=1}}^{ \rev{N_{B}(T+1)} } \{ \underline{\mu}^i_{B} 
	\notin [\underline{L}^i_{B}(t)  , \underline{U}^i_{B}(t)]  \} \right) . \notag
	\end{align}
	Both terms on the right-hand side of the inequality above can be bounded using the concentration inequality in \rev{Appendix A}. \rvs{Using 
		$\theta = \delta / (2 d_r T)$}, in \rev{Appendix A} gives
	$\Pr( \text{UC}^i_{B} | {\cal B}(T)) \leq \delta /( d_r T)$,
	since \rev{$1 + N_{B}(t) \leq 1 + N_B(T+1) \leq 2T$}. 
	Then, using the union bound over all objectives, we obtain
	$\Pr( \text{UC}_B | {\cal B}(T)) \leq \delta / T$. 
	
	\subsection*{APPENDIX C - PROOF OF LEMMA 5}	
	
	The maximum number of times a radius $r$ ball $B$ can be selected before it becomes a parent ball is upper bounded by $1 + 2 r^{-2} (1 +  2 \log (   2 \sqrt{2} d_r T^{\frac{3}{2}} / \delta) )$. From the result of Lemma 3, we know that the Pareto regret in each of these rounds is upper bounded by $14 r$. Note that after ball $B$ becomes a parent ball, it will create a new radius $r/2$ child ball every time it is selected. 
	From Lemma 4, we know that the Pareto regret in each of these rounds is bounded above by $12 (r/2)$. Therefore, we can include the Pareto regret incurred in a round in which a new child ball with radius $r$ is created from a parent ball as a part of the child ball's (total) Pareto regret. Hence, the Pareto regret incurred in a radius $r$ ball is upper bounded by 
	\begin{align*}
	& 14r \left( 1 + 2 r^{-2} (1 +  2 \log (   2 \sqrt{2} d_r T^{\frac{3}{2}} / \delta) ) \right) + 12r \\
	&\leq 14 r \left( 2 + 2 r^{-2} (1 +  2 \log (   2 \sqrt{2} d_r T^{\frac{3}{2}} / \delta) ) \right) \\
	&\leq 56 r^{-1} \log (   2 \sqrt{2} d_r T^{\frac{3}{2}} e / \delta) .
	\end{align*}
	
	Let $r_l := 2^{\ceil*{\log (r_0) / \log (2)} }$. We have $r_0 /2 \leq r_l /2 \leq r_0 \leq r_l \leq 2 r_0$. The one-round Pareto regret of the balls whose radii are smaller than $r_l$ is bounded by $14 r_l$ on event $\text{UC}^c$ according to Lemma 3. Also, we know that $14 r_l \leq 28 r_0$ by the above inequality. Therefore, the Pareto regret due to all balls with radii smaller than $r_l$ by time $T$ is bounded by $28 T r_0$, and the Pareto regret due to all balls with radii $r = 2^{-i} \geq r_0$ is bounded by $56 r^{-1} N_r \log (   2 \sqrt{2} d_r T^{\frac{3}{2}} e / \delta)$. Thus, summing this up for all possible balls, we obtain
	the following Pareto regret bound on event $\text{UC}^c$:
	\begin{align}
	\text{Reg}(T) &\leq 28 T r_0 \notag \\
	&+ \sum_{r = 2^{-i}: i \in \mathbb{N}, r_0 \leq r \leq 1} 56 r^{-1} N_r \log (  2 \sqrt{2} d_r T^{\frac{3}{2}} e / \delta) . \notag 
	\end{align}

	\subsection*{APPENDIX D- SIMULATIONS}
	
	We evaluate the performance of PCZ on a synthetic dataset. We take ${\cal X} = [0,1]$, ${\cal Y} = [0,1]$, and generate $\mu^1_y(x)$ and $\mu^2_y(x)$ as shown in Figure \ref{fig:p1}. 
	To generate $\mu^1_y(x)$, we first define a line by equation $8x+10y=8$ and let $y_1(x) = (8-8x)/10$. For all context arm pairs $(x,y)$, we set $\mu^1_{y}(x) = \max \{ 0, (1-5|y - y_1(x)|) \}$.
	Similarly, to generate $\mu^2_y(x)$, we define the line $8x + 10y = 10$ and let $y_2(x) = (10-8x)/10$. 
	Then, we set $\mu^2_y(x) = \max \{ 0, (1-5(y_2(x) - y)) \}$ for $y \leq y_2(x)$ and 
	$\mu^2_y(x) = \max \{ 0, (1-(y - y_2(x))/4) \}$ for $y > y_2(x)$.
	
	Based on the definitions given above, the Pareto optimal arms given context $x$ lie in the interval $[y_1(x), y_2(x)]$. 
	To evaluate the fairness of PCZ, we define six bins that correspond to context-arm pairs in the Pareto front. Given context $x$, the $1$st bin contains all arms in the interval $[y_1(x), y_1(x) + 1/30]$ and the $i$th bin $i \in \{2,\ldots,6\}$ contains all arms in the interval $(y_1(x) + (i-1)/30, y_1(x) + i/30]$.
	Simply, the first three bins include the Pareto optimal arms whose expected rewards in the first objective are higher than the expected rewards in the second objective and the last three bins include the Pareto optimal arms whose expected rewards in the second objective are higher than the expected rewards in the first objective.
	
	We assume that the reward of arm $y$ in objective $i$ given context $x$ is a Bernoulli random variable with parameter $\mu^i_y(x)$. In addition, at each round $t$, the context $x_t$ is sampled from the uniform distribution over ${\cal X}$. 
	
	\begin{figure}[htb] 
		\begin{minipage}[b]{1.0\linewidth}
			\centering
			\centerline{\includegraphics[width=6.55cm]{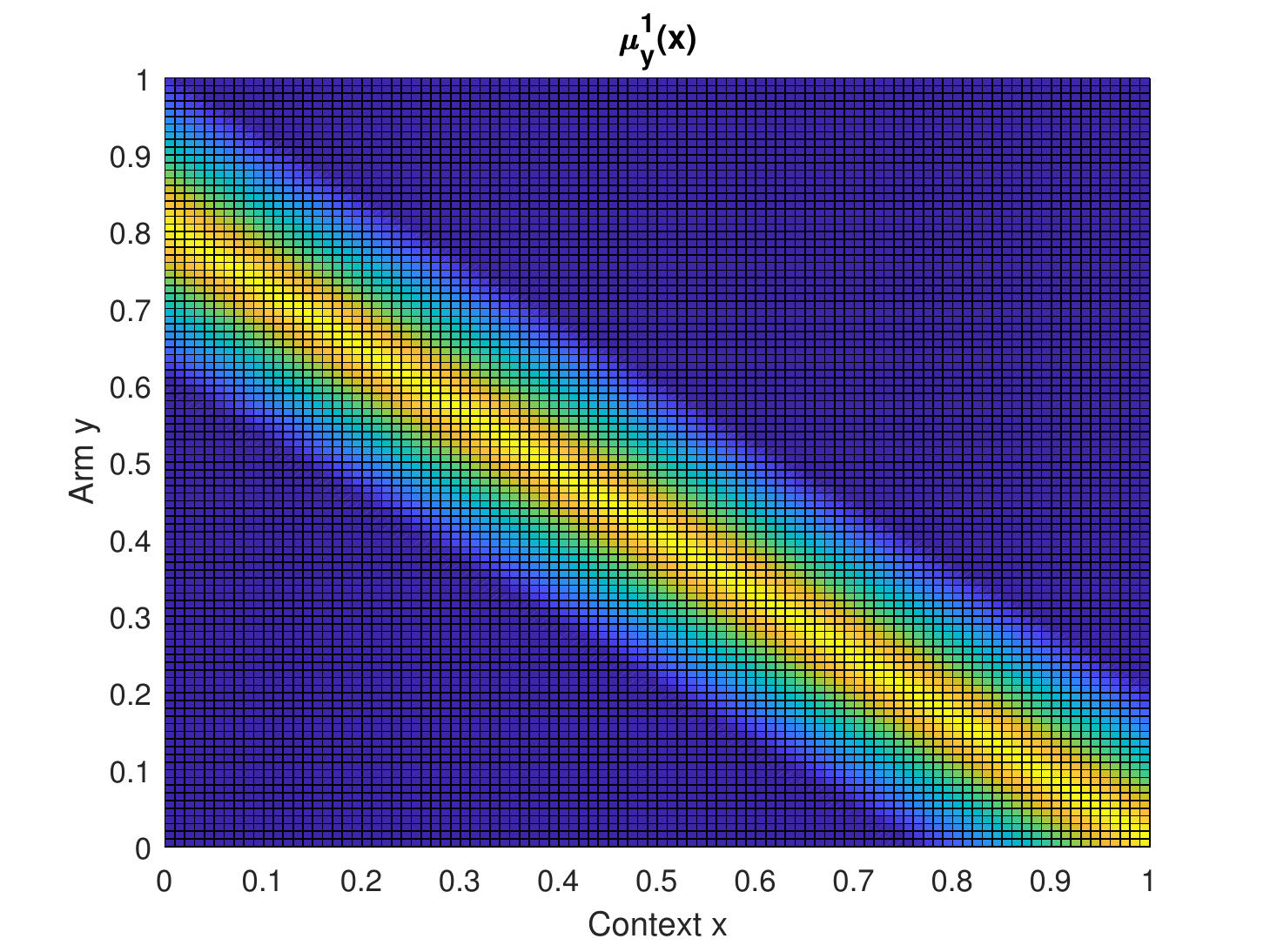}}
			
		\end{minipage}
		
		\begin{minipage}[b]{1.0\linewidth}
			\centering
			\centerline{\includegraphics[width=6.55cm]{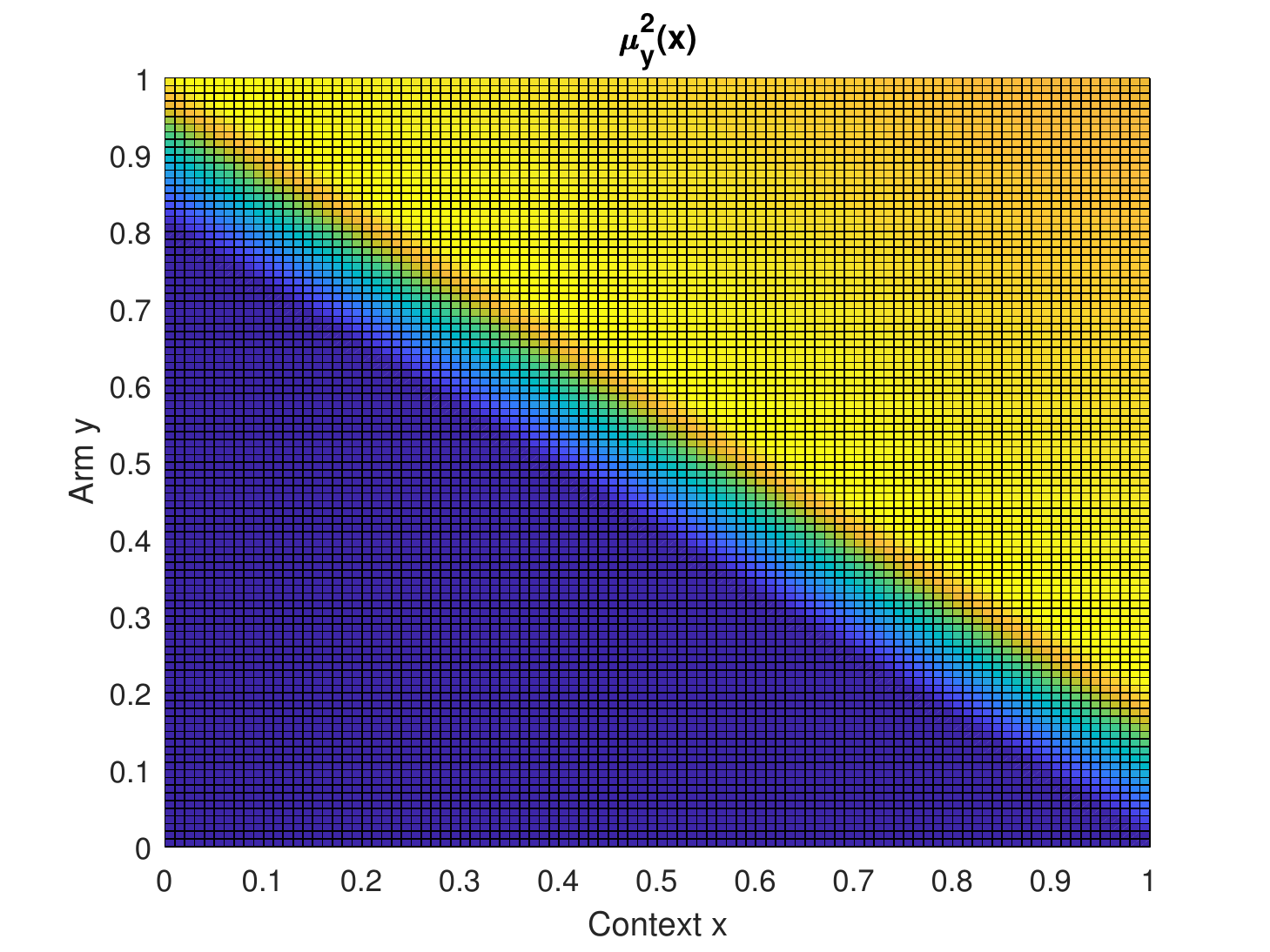}}
			
		\end{minipage}
		
		\caption{Expected Rewards of Context-Arm Pairs (Yellow Represents $1$, Dark Blue Represents $0$)}
		%
		\label{fig:p1}
	\end{figure}

	\begin{figure}[htb] 
		\begin{minipage}[b]{1.0\linewidth}
			\centering
			\centerline{\includegraphics[width=6.55cm]{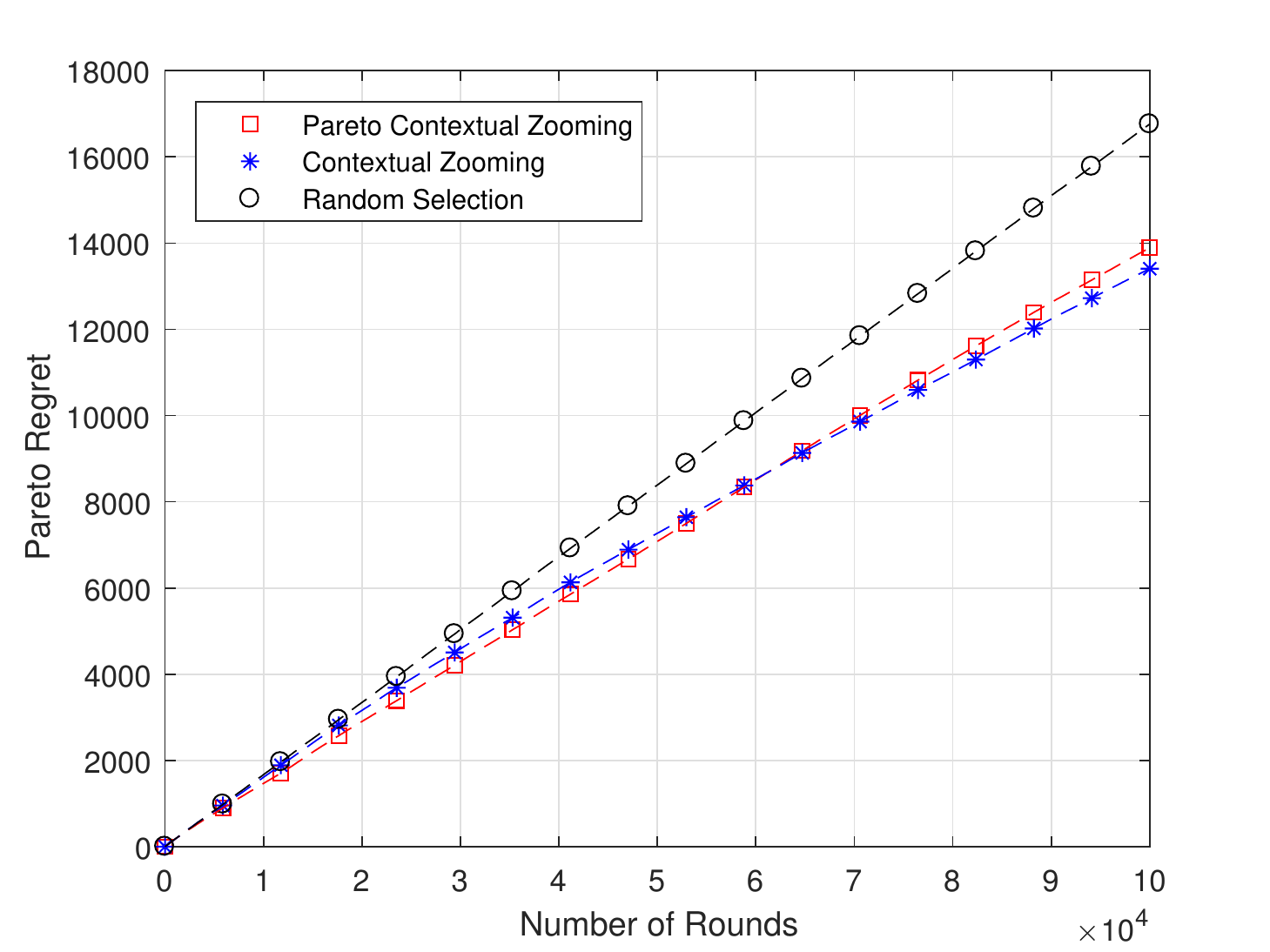}}
			
		\end{minipage}
		
		\begin{minipage}[b]{1.0\linewidth}
			\centering
			\centerline{\includegraphics[width=6.55cm]{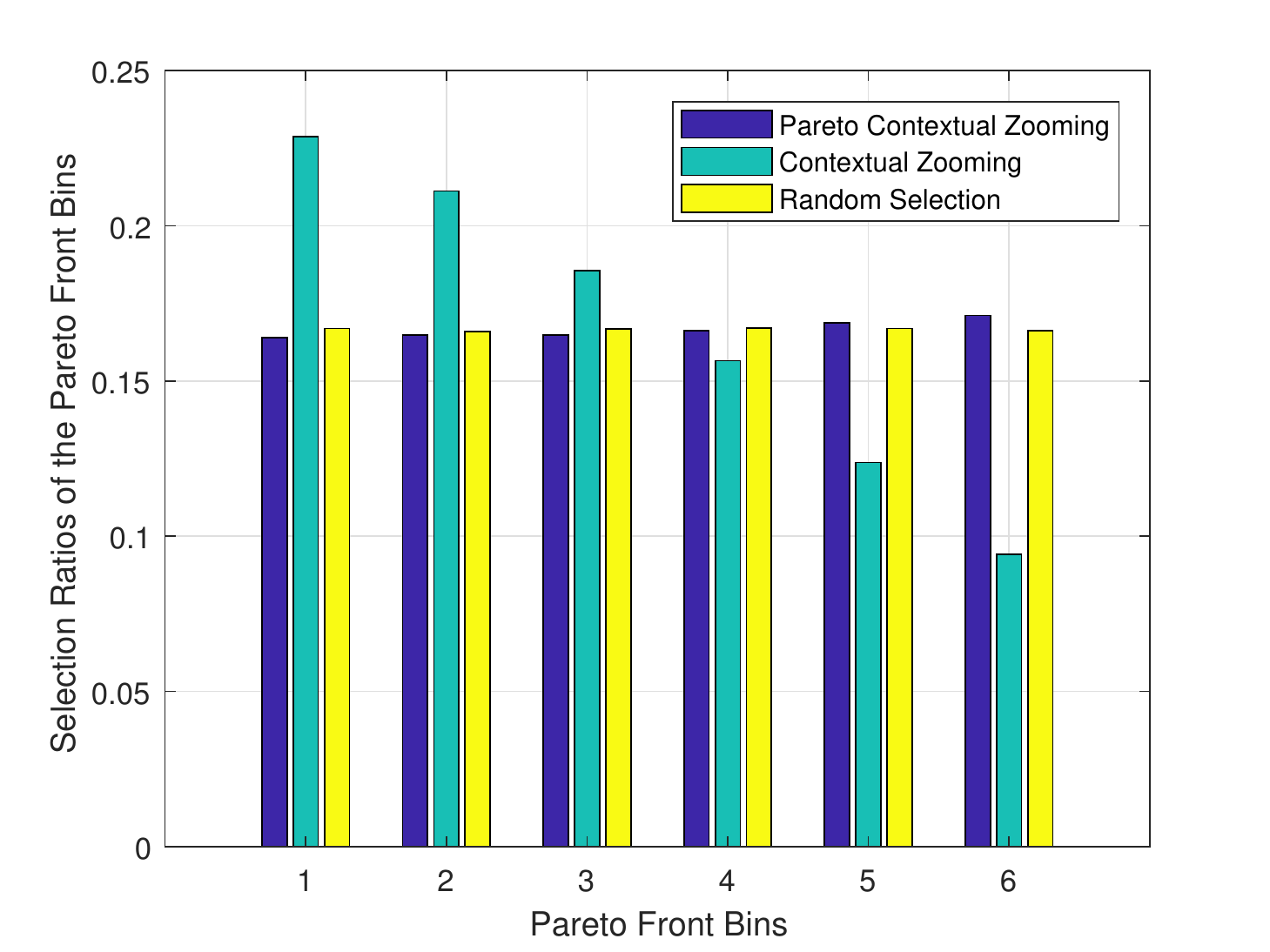}}
			
		\end{minipage}
		
		\caption{(i) Pareto Regret vs. Number of Rounds (ii) Selection Ratio of the Pareto Front Bins}
		%
		\label{fig:p2}
	\end{figure}
	
	We compare our algorithm with Contextual Zooming \cite{slivkins2014contextual} and Random Selection, which chooses in each round an arm uniformly at random from ${\cal Y}$. Contextual Zooming only uses the rewards in the first objective to update itself. \revn{Both PCZ and Contextual Zooming uses scaled Euclidean distance.\footnote{We set $D((x,y),(x',y')) = \sqrt{(x-x')^2 + (y-y')^2}/\sqrt{2}$. While this choice does not satisfy Assumption 1, we use this setup to illustrate that learning is still possible when the distance function is not perfectly known by the learner.}} We choose $\delta = 1/T$ in PCZ, set $T=10^5$, run each algorithm $100$ times, and report the average of the results in these runs. 
	
	The Pareto regret is reported in Figure \ref{fig:p2}(i) as a function of the number of rounds. It is observed that the Pareto regret of PCZ at $T=10^5$ is $3.61 \%$ higher than that of Contextual Zooming and $17.1 \%$ smaller than that of Random Selection. We compare the fairness of the algorithms in Figure \ref{fig:p2}(ii). For this, we report the selection ratio of each Pareto front bin, which is defined for bin $i$ as the number of times a context-arm pair in bin $i$ is selected divided by the number of times a Pareto optimal arm is selected by round $T$. We observe that the selection ratio of all bins are almost the same for PCZ, while Contextual Zooming selects the context-arm pairs in the $1$st bin much more than the other bins.

\end{document}